\documentclass{article}

\usepackage{log_2026}                    %

\usepackage{graphicx}
\usepackage{booktabs}
\usepackage{multirow}
\usepackage{tikz}
\usetikzlibrary{arrows.meta}

\usepackage{amssymb}
\usepackage{amsthm}
\usepackage{thmtools}

\makeatletter
\IfFormatAtLeastTF{2026-06-01}{%
  \IfPackageAtLeastTF{thmtools}{2023/05/05}{}{%
    \renewcommand\thmt@autorefsetup{%
      \@xa\def\csname\thmt@envname autorefname\@xa\endcsname\@xa{\thmt@thmname}%
    }%
  }%
}{}
\makeatother

\usepackage[numbers,compress,sort]{natbib}
\AtEndPreamble{%
  \RequirePackage[capitalize,noabbrev]{cleveref}%
}

\newif\ifdraft
\draftfalse  %

\newcommand{\RR}{\mathbb{R}}

\newcommand{\pr}{\mathbb{P}}
\newcommand{\E}{\mathbb{E}}
\newcommand{\ra}{\rightarrow}

\newcommand*{\QED}{\null\nobreak\hfill\ensuremath{\square}}
\newcommand{\diag}{\mathsf{diag}}
\newcommand{\trace}{\mathsf{trace}}
\newcommand{\Var}{\mathsf{Var}}

\DeclareMathOperator{\MLP}{MLP}

\theoremstyle{plain}
\newtheorem{theorem}{Theorem}[section]
\newtheorem{lemma}[theorem]{Lemma}
\newtheorem{corollary}[theorem]{Corollary}

\newcommand{\veryshortarrow}[1][5pt]{\mathrel{%
   \hbox{\rule[\dimexpr\fontdimen22\textfont2-.2pt\relax]{#1}{.4pt}}%
   \mkern-4mu\hbox{\usefont{U}{lasy}{m}{n}\symbol{41}}}}
\newcommand{\Cd}[2]{C^{\veryshortarrow\!}_{#1}(#2)}
\newcommand{\SC}{\text{SC}}

\title{Random Probing for Structural Self-Interactions in\\ Graph Neural Networks}

\author[Peltonen et al.]{%
Saku Peltonen\\
ETH Zurich \\
\email{speltonen@ethz.ch}\And
H. Çağrı Bilgi\\
TU Delft\\
\email{h.c.bilgi@tudelft.nl}\\
\And
Kubilay Atasu\\
TU Delft\\
\email{kubilay.atasu@tudelft.nl}
}

\begin{document}

\maketitle

\begin{abstract}
    Graph Neural Networks (GNNs) fundamentally lack the ability to capture structural properties defined by node self-interactions, such as cycle counts, subgraph centrality, and heat kernel signatures. Existing solutions often rely on computationally expensive higher-order architectures or hand-crafted features. In this work, we introduce DIAL, a message-passing layer that gives nodes access to graph structure through the diagonal of graph-derived operators. Our approach uses randomized probing to provide nodes with learnable, permutation equivariant access to diagonal entries.  We prove that our approach can express important structural features, including cycle counts, subgraph centrality, and heat kernel signatures.  The proposed method can be integrated directly into message passing layers. Experiments on molecular property prediction, cycle counting, and pattern detection demonstrate the effectiveness of DIAL across a range of graph learning tasks.
\end{abstract}

\section{Introduction}

Graph Neural Networks (GNNs) have been successful in many applications. However, they are known to have limited expressive power, often failing to distinguish basic structures or count fundamental motifs like cycles \cite{xu2018powerful}. Higher-order GNN architectures can partially address these limitations, but often at a significant computational cost \cite{morris2019weisfeiler}.
In practice, one often resorts to hand-crafted structural features, such as cycle or motif counts \cite{bouritsas2022improving}. While effective, this approach is brittle and requires choosing the relevant structures in advance.

Many structural node properties that message passing fails to capture can be described in terms of \textit{self-interactions} of nodes. 
Intuitively, these interactions quantify how information originating from a node flows through the graph and returns to it.
Standard GNNs cannot capture such patterns—for example closed-walk participation—as they require isolating information that originates from a node and later returns to it.
This concept naturally generalizes to subgraph centrality, which characterizes a node's global importance by weighing its closed walk participation.
In the spectral domain, Laplacian-based self-interactions cover graph angles (Figure~\ref{fig:graph_angle_structural_encoding}), as well as diffusion dynamics like the heat kernel signature, which measures heat return rates over time.

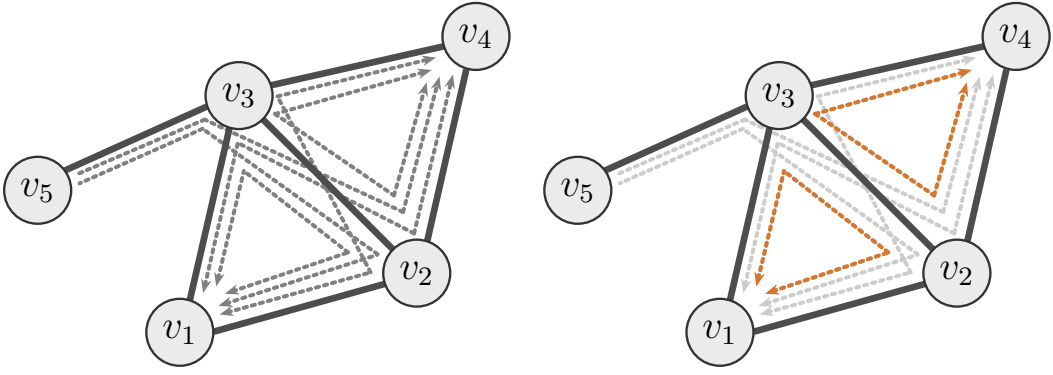
\begin{figure}[t]
    \centering
    \begingroup
\definecolor{hutchnodebackground}{RGB}{235,235,235}
\colorlet{hutchnodeborder}{black!80}
\colorlet{hutchgraphedge}{black!70}
\colorlet{hutchmessagepath}{black!50}
\definecolor{hutchmessagehighlight}{RGB}{215,118,45}
\colorlet{hutchmessagemuted}{black!20}

\newcommand{\HutchNodeOutlineWidth}{0.7pt}
\newcommand{\HutchGraphEdgeWidth}{2pt}
\newcommand{\HutchMessagePathWidth}{1.2pt}
\newcommand{\HutchMessageArrowScale}{0.5}
\newcommand{\HutchFormulaX}{9.05}
\newcommand{\HutchFormulaY}{7.58}

\newcommand{\DrawHutchNodes}{%
  \draw[graphnode] (8.25,8.71) circle (0.3535533905932738cm)
    node[graphnodelabel] {$v_1$};
  \draw[graphnode] (8.875,11.21) circle (0.3535533905932738cm)
    node[graphnodelabel] {$v_3$};
  \draw[graphnode] (10.75,9.335) circle (0.3535533905932738cm)
    node[graphnodelabel] {$v_2$};
  \draw[graphnode] (11.375,11.835) circle (0.3535533905932738cm)
    node[graphnodelabel] {$v_4$};
  \draw[graphnode] (6.75,10.21) circle (0.3535533905932738cm)
    node[graphnodelabel] {$v_5$};
}

\newcommand{\DrawHutchGraphEdges}{%
  \draw[graphedge] (7.038407621333391,10.419409822500029) -- (8.532472141999978,11.093118035499995);
  \draw[graphedge] (9.212115834333359,11.31876392900001) -- (11.046262205122156,11.733131102561078);
  \draw[graphedge] (8.352779815729232,9.036856267413338) -- (8.775325208853474,10.872987813280211);
  \draw[graphedge] (8.60624814710916,8.71625061763028) -- (10.419232885737152,9.223009672275632);
  \draw[graphedge] (10.86060557383581,9.672985761292084) -- (11.273131102561079,11.485444212817187);
  \draw[graphedge] (10.5,9.585) -- (9.125,10.96);
}

\newcommand{\DrawHutchHighlightedMessages}[1]{%
  \draw[#1backward] (8.633554839632314,9.1936548436884) -- (8.925699656505483,10.413591734060642);
  \draw[#1] (10.025926138805337,9.55344980098314) -- (8.925699656505483,10.413591734060642);
  \draw[#1backward] (8.725678068901274,9.116123769590807) -- (10.025926138805337,9.55344980098314);
  \draw[#1forward] (10.509410435627945,10.15575881609538) -- (10.854534777277964,11.350162155963739);
  \draw[#1forward] (9.254676010538902,11.016582090677774) -- (10.907702327342545,11.48111635659996);
  \draw[#1] (9.254676010538902,11.016582090677774) -- (10.509410435627945,10.15575881609538);
}

\newcommand{\DrawHutchOtherMessages}[1]{%
  \draw[#1backward] (8.495720562403982,9.123181581867124) -- (8.86255039711051,10.71411489953799);
  \draw[#1] (10.721527701390746,9.75466970477229) -- (8.86272712765304,10.709922772173998);
  \draw[#1forward] (10.721527701390746,9.75466970477229) -- (11.13218625975211,11.434249290510756);
  \draw[#1forward] (9.279717335663314,11.2019506874419) -- (10.953885190210585,11.612550360052282);
  \draw[#1] (10.263251373950125,9.343324963767907) -- (9.279717335663314,11.2019506874419);
  \draw[#1backward] (8.668649130570826,8.897020049923123) -- (10.263251373950125,9.343324963767907);
  \draw[#1] (8.51266685208021,10.957036034179117) -- (7.160897685032687,10.376885260128834);
  \draw[#1] (10.605685293779908,9.980564150382438) -- (8.51266685208021,10.957036034179117);
  \draw[#1forward] (10.605685293779908,9.980564150382438) -- (10.990444442586032,11.412352840930154);
  \draw[#1] (8.490269413939961,10.859592229312398) -- (7.192182354852323,10.28297955152153);
  \draw[#1] (10.337278325586944,9.534034382431186) -- (8.490269413939961,10.859592229312398);
  \draw[#1backward] (8.655964144800112,8.988216445480099) -- (10.330635492321075,9.527606513776723);
}

\newcommand{\DrawHutchAllMessages}[1]{%
  \DrawHutchHighlightedMessages{#1}
  \DrawHutchOtherMessages{#1}
}

\newcommand{\DrawHutchPropagationFormula}{%
  \node[formulabox] at (\HutchFormulaX,\HutchFormulaY)
    {$\displaystyle (A^3z)_i=\sum_j (A^3)_{ij}z_j$};
}

\newcommand{\DrawHutchEstimatorFormula}{%
  \node[formulabox] at (\HutchFormulaX,\HutchFormulaY)
    {$\displaystyle
      \mathbb{E}[z_iz_j]=0\ (i\ne j),\quad \mathbb{E}[z_i^2]=1$\\
     $\displaystyle
      \Rightarrow\ \mathbb{E}[z_i(A^3z)_i]=(A^3)_{ii}$};
}

\newcommand{\DrawHutchPanel}[1]{%
  \begin{scope}[#1]
    \DrawHutchAllMessages{messagepath}
    \DrawHutchGraphEdges
    \DrawHutchNodes
  \end{scope}
}

\newcommand{\DrawHutchHighlightedPanel}[1]{%
  \begin{scope}[#1]
    \DrawHutchAllMessages{mutedmessagepath}
    \DrawHutchHighlightedMessages{highlightedmessagepath}
    \DrawHutchGraphEdges
    \DrawHutchNodes
  \end{scope}
}

\resizebox{\linewidth}{!}{%
\begin{tikzpicture}[
  every node/.style={font=\large},
  graphnode/.style={
    draw=hutchnodeborder,
    fill=hutchnodebackground,
    line width=\HutchNodeOutlineWidth
  },
  graphnodelabel/.style={
    inner xsep=0.080cm,
    inner ysep=0.085cm,
    rounded corners=0.020cm
  },
  graphedge/.style={
    draw=hutchgraphedge,
    line width=\HutchGraphEdgeWidth
  },
  messagedash/.style={
    dash pattern=on 1pt off 1.5pt,
    dash phase=0pt
  },
  messagepath/.style={
    draw=hutchmessagepath,
    line width=\HutchMessagePathWidth,
    messagedash,
    line cap=round,
    line join=round
  },
  highlightedmessagepath/.style={
    draw=hutchmessagehighlight,
    line width=\HutchMessagePathWidth,
    messagedash,
    line cap=round,
    line join=round
  },
  mutedmessagepath/.style={
    draw=hutchmessagemuted,
    line width=\HutchMessagePathWidth,
    messagedash,
    line cap=round,
    line join=round
  },
  messagepathforward/.style={
    messagepath,
    -{Stealth[scale=\HutchMessageArrowScale]}
  },
  messagepathbackward/.style={
    messagepath,
    {Stealth[scale=\HutchMessageArrowScale]}-
  },
  highlightedmessagepathforward/.style={
    highlightedmessagepath,
    -{Stealth[scale=\HutchMessageArrowScale]}
  },
  highlightedmessagepathbackward/.style={
    highlightedmessagepath,
    {Stealth[scale=\HutchMessageArrowScale]}-
  },
  mutedmessagepathforward/.style={
    mutedmessagepath,
    -{Stealth[scale=\HutchMessageArrowScale]}
  },
  mutedmessagepathbackward/.style={
    mutedmessagepath,
    {Stealth[scale=\HutchMessageArrowScale]}-
  },
  formulabox/.style={
    font=\normalsize,
    align=center,
    text=black!88,
    fill=white,
    draw=black!18,
    rounded corners=2pt,
    inner xsep=5pt,
    inner ysep=4pt
  }
]

\DrawHutchPanel{}
\DrawHutchHighlightedPanel{xshift=5.7cm}

\end{tikzpicture}
}
\endgroup
    \caption{Message-passing view of Hutchinson's diagonal estimator. We illustrate triangle counting for all nodes simultaneously using random probes and ordinary message passing. \textbf{Left:} The arrowed paths show a subset of length-3 messages received at $v_1$ and $v_4$. After propagation, node $v$ observes $(A^3 z)_v=\sum_u (A^3)_{vu}z_u$, which mixes self-returning messages with messages originating elsewhere. Without identifying information, standard message passing cannot isolate the diagonal term. \textbf{Right:} Hutchinson's estimator correlates the propagated value with the original probe $z_v$. Since $\E[z_vz_u]=0$ for $u\neq v$ and $\E[z_v^2]=1$, the non-diagonal terms vanish in expectation, yielding $\E[z_v(A^3z)_v]=(A^3)_{vv}$. For an adjacency matrix, this diagonal entry counts length-3 closed walks at $v$, i.e., twice the number of triangles incident to $v$.}
    \label{fig:hutchinson_message_passing}
\end{figure}

Mathematically, these notions can all be expressed as functions of a Graph Shift Operator (GSO) $S$, such as the adjacency matrix $A$ or Laplacian $L$. 
Crucially, all such node-level self-interactions correspond to the \textit{diagonal entries} of these GSO-derived operators $M = f(S) \in \mathbb{R}^{n \times n}$.
In practice, $f$ is often approximated by low-degree polynomials of the GSO, $M=\sum_{k=0}^{K-1} \alpha_k S^k$. This naturally aligns with message-passing architectures, which implicitly compute GSO polynomials through neighborhood aggregation. However, standard message-passing GNNs cannot express any of the self-interaction-based structural properties above. 
The aggregation mechanism mixes information from different nodes, failing to isolate self-interactions. 

Our key observation is that message passing algorithms can access diagonal entries in a permutation equivariant way using a randomized probing mechanism. With a stochastic approach based on the Hutchinson diagonal estimator \cite{Hutchinson01011990}, nodes can efficiently approximate the diagonal of GSO powers, and hence the diagonal of any graph-derived operator, through simple message passing steps. This mechanism allows nodes to capture structural properties, such as cycle counts, within a fully learnable framework, avoiding rigid hand-crafted features and expensive higher-order architectures. See Figure~\ref{fig:hutchinson_message_passing} for an example of triangle counting. 

The resulting architecture, which we call DIAL, can be integrated into any standard GNN and is permutation equivariant in expectation. Adjusting the number of random probes allows a trade-off between computational cost and estimation accuracy. We prove DIAL can express a variety of important structural properties, including cycle counts, graph angles, subgraph centrality, and heat kernel signatures.

Empirically, we evaluate DIAL on molecular property prediction on ZINC and on synthetic node-level pattern counting tasks.
The results show that DIAL is competitive with spectral and random-probe positional encoding methods on ZINC, and improves over other random-probe methods on cycle counting while remaining substantially cheaper than specialized subgraph and higher-order architectures.
\footnote{The implementation of DIAL is available at \url{https://github.com/hcagri/DIAL}}

\section{Preliminaries}

A graph is denoted $G=(V,E)$ with $|V|=n$. A graph shift operator (GSO) is a matrix $\mathbf S \in \RR^{n \times n}$. Common choices include the adjacency matrix $\mathbf A$, normalized adjacency $\mathbf D^{-1/2} \mathbf A \mathbf D^{-1/2}$, the Laplacian $\mathbf L = \mathbf D - \mathbf A$, where $\mathbf D$ is the degree matrix, or the normalized Laplacian. We consider message-passing layers of the form
\begin{equation}
    \label{eq:generalized_GNN}
    \mathbf X^{(\ell)} = \sigma\Big(\sum_{k=0}^{K-1} \mathbf S^k \mathbf X^{(\ell-1)} \mathbf H_k^{(\ell)}\Big),
\end{equation}
where $\sigma$ is a nonlinearity and $\mathbf H_k^{(\ell)}$ are learnable weights.
The parameter $K$ controls the maximum hop distance used for aggregation (specifically $K-1$), with $K=2$ for typical GNNs such as GCN.

In what follows, we assume the graph is undirected, making the GSO symmetric. Let $\mathbf S = \Phi \Lambda \Phi^{-1}$ be the eigendecomposition of the GSO $\mathbf S$, where $\Lambda = \diag(\lambda_1, \dots, \lambda_n)$ is the diagonal matrix of eigenvalues and $\Phi = [\boldsymbol\phi_1, \dots, \boldsymbol\phi_n]$ is the matrix of eigenvectors. For symmetric GSOs (undirected graphs), $\Phi$ is orthonormal, i.e. $\Phi^{-1} = \Phi^T$.

Each eigenvalue $\mu$ defines an eigenspace $E_\mu$ spanned by the eigenvectors with eigenvalue $\mu$. Let $I_\mu = \{i: \lambda_i = \mu\}$ denote the corresponding index set and let $\Phi_\mu \in \RR^{n \times |I_\mu|}$ be the matrix whose columns are these eigenvectors. The projector onto $E_\mu$ is $P_\mu = \Phi_\mu \Phi_\mu^T$. These projections allow the GSO to be expressed as a sum over eigenspaces: $S = \sum_\mu \mu P_\mu$, where the sum is over distinct eigenvalues.

Graph angles describe how nodes relate to these eigenspaces.
For a node $v$, let $\mathbf e_v$ denote its indicator vector.
The graph angle of $v$ with respect to $E_\mu$ is
\[
    \angle_\mu(v) = (P_\mu)_{vv} = \|P_\mu \mathbf e_v\|^2 = \sum_{i \in I_\mu} (\boldsymbol\phi_i(v))^2.
\]
The term \textit{graph angle} comes from the geometric interpretation
\[
    (P_\mu)_{vv} = \cos^2(\theta_{v,\mu}),
    \qquad
    \theta_{v,\mu} = \angle(\mathbf e_v, E_\mu).
\]
Unlike individual eigenvector coordinates, graph angles are invariant to sign changes and changes of basis inside eigenspaces. These invariances remove arbitrary choices from the representation and can therefore simplify learning \cite{lim2022sign}.

Graph filters are operators of the form $p(S)$, where $p: \RR \ra \RR$ is applied to the GSO $S$.
For symmetric GSOs, spectral filters can be written as
\[
    p(S) = \Phi p(\Lambda) \Phi^T = \sum_\mu p(\mu) P_\mu.
\]
A polynomial filter is the special case where $p(\mu) = \sum_{k=0}^{K-1} a_k \mu^k$, equivalently
\[
    p(S) = \sum_{k=0}^{K-1} a_k S^k.
\]
The linear propagation part of Equation~(\ref{eq:generalized_GNN}) naturally implements polynomial filters of the GSO.

\section{Related Work}

\paragraph{Expressivity of GNNs.}
Standard message passing architectures are limited in their expressiveness. In particular, they are at most as powerful as the Weisfeiler-Leman test \citep{xu2018powerful} in distinguishing non-isomorphic graphs. This limitation surfaces in the inability of message-passing GNNs to count many basic substructures \citep{chen2020can}. For example, standard GNNs cannot count triangles \citep{morris2019weisfeiler} or longer cycles, despite the practical relevance of this information in several applications, such as molecular property prediction \cite{wollschlager2024expressivity}. From the perspective of this work, this limitation stems from the fact that message passing aggregates information arriving from neighboring nodes, but does not keep track of where that information originally comes from. Consequently, message passing cannot isolate self-interactions such as cycles. 

There have been many works proposing more expressive GNN architectures \citep{morris2019weisfeiler,maron2019provably,chen2020can,papp2021dropgnn,vignac2020building}. Higher-order GNNs operate on tuples of nodes, making richer structural interactions available to the model \cite{morris2019weisfeiler}. However, this comes at a cost growing polynomially with the tuple order. Other approaches increase expressivity by explicitly augmenting node representations with structural information. Substructure-counting methods add counts of selected substructures \cite{bouritsas2022improving,chen2020can}, but require selecting the relevant substructures in advance. Distance encodings instead assign each node an encoding of its distance to a target node set \cite{li2020distance}, conditioning the resulting representation on that set. In contrast to these approaches, our method exposes self-interaction information through randomized diagonal estimation. This allows a standard message passing backbone to access structural information without higher-order tuples, a preselected substructure vocabulary, or target-conditioned encodings.

\paragraph{Random Features in GNNs.}
Random node features have been studied as a way to overcome expressivity limitations in message passing \citep{dasoulas2019coloring,abboud2020surprising,papp2022theoretical}. These works show that random initialization can break node symmetries and increase expressive power, with universality results in some settings \citep{puny2020globalattentionimprovesgraph}. 
However, this relies on the model learning to use randomness as node identifiers, and the theoretical power does not directly provide stable, task-relevant structural quantities. Hence, more direct uses of random features have been proposed \cite{kanatsoulis2024counting,kanatsoulis2025learning,eliasof2023graph}.

R-PEARL \citep{kanatsoulis2025learning} introduces a randomized variant of a positional encoding framework. While effective for PEs, its expressivity guarantees are limited.\footnote{The expressivity guarantees come mainly from \citep{kanatsoulis2024counting}, which uses random features for substructure counting. We defer a more detailed discussion to Appendix~\ref{app:kanatsoulis}.} 

\citet{eliasof2023graph} use message passing with random features to learn positional encodings. They show that, combined with appropriate normalization, random feature propagation simulates power iteration and subspace iteration. This allows approximating the leading eigenvectors of the graph shift operator. They also note that random features can be used to estimate cycle counts using Hutchinson's trace estimator. However, the focus of their work is on positional encodings, and their analysis is based on what can be achieved in the limit as the propagation depth increases. In contrast, we focus on structural self-interactions and explicitly correlate propagated signals with the input randomness. This correlation is essential: it yields an unbiased estimator of diagonal entries in expectation and provides a stable structural signal that can be learned from finite random probes.

\paragraph{Spectral Graph Theory and GNNs.}

Spectral approaches to graph learning are rooted in spectral graph theory and graph signal processing \citep{chung1997spectral}. Graph signal processing interprets node features as signals on the graph and defines filters as functions of a graph shift operator \citep{shuman2013emerging}. Classical spectral filters are typically fixed or analytically chosen. 

Spectral GNNs learn such filters from data. Early spectral GNNs implemented filters directly in the Fourier basis, requiring an eigendecomposition \citep{bruna2013spectral}. ChebNet \citep{defferrard2016convolutional} introduced Chebyshev polynomial approximations. GCN can be viewed as a first-order approximation of this spectral filtering perspective~\cite{kipf2016semi}. These methods apply graph-derived operators to node features, but the resulting node representations do not explicitly expose the diagonal entries of those operators. Our method complements spectral GNNs by showing how randomized probing can make these diagonal entries accessible to message passing.

\section{Main Tool: Diagonal Estimation via Message Passing}
\label{sec:diag_estimate}
This section presents our main technical tool: a randomized technique to estimate the diagonal entries of a matrix via message passing. 
It is based on Hutchinson's diagonal estimator \cite{Hutchinson01011990}, which is formalized in the next lemma:

\begin{restatable}{lemma}{diagonalEstimator}
    \label{lem:hutchinson_diagonal_estimator}
    Let $M \in \RR^{n \times n}$. Let $z \in \RR^n$ be a random vector with i.i.d. Rademacher entries (i.e., each entry is independently $+1$ or $-1$ with probability $1/2$). Then 
    $$\E[z \odot (Mz)] = \diag(M)$$
    where $\odot$ is the elementwise product. The variance of the $i$th entry is $\sum_{j \neq i} M_{ij}^2$.
\end{restatable}
\noindent \textit{Proof.} The $i$th entry is $\E[z_i (Mz)_i] = \E[\sum_j M_{ij} z_i z_j]$. Here, $\E[z_i z_j]=0$ for $i \neq j$ since $\E[z_i] = 0$ and $z_i, z_j$ are independent, making all cross terms vanish. Thus, only the $j=i$ term remains, giving $\E[z_i^2] M_{ii} = M_{ii}$. 

The variance is computed from $\Var(d_i) = \E[d_i^2] - (\E[d_i])^2$. Here, 
$$d_i^2 = (\sum_j M_{ij} z_i z_j)(\sum_j M_{ij} z_i z_j) = \sum_{j,k} M_{ij} M_{ik} z_i^2 z_j z_k = \sum_{j,k} M_{ij} M_{ik} z_j z_k$$ 
because $z_i^2 = 1$. Taking the expectation, $\E[d_i^2] = \E\Big[\sum_{j,k} M_{ij} M_{ik} z_j z_k\Big] =  \sum_j M_{ij}^2$, since $\E[z_j z_k] = 0$ for $j \neq k$. Thus,
$\Var(d_i) = \E[d_i^2] - (\E[d_i])^2 = \sum_j M_{ij}^2 - M_{ii}^2$
\QED

In the context of GNNs, the matrix $M$ can be the GSO or a matrix derived from it, such as powers or spectral filters. $Mz$ corresponds to a message passing operation on the graph with input features $z$.

\subsection{Warm-up: Cycle Counting via Diagonal Estimation}

Consider a message passing model of the form 
\begin{equation}
    \label{eq:MPNN}
    X^{(l)} = \sigma(A X^{(l-1)} W)    
\end{equation}
Ignoring nonlinearities and weights, after $k$ layers, we have $X^{(k)} = A^k X^{(0)}$. We illustrate how diagonal estimation can be used to compute cycle counts of the graph:

\begin{lemma}
    With a GNN in Equation~\ref{eq:MPNN} and uniformly random initial features $X^{(0)} \in \{-1,1\}^n$, each node $v$ can compute the number of closed walks of length $k$ starting and ending at $v$ in expectation.
\end{lemma}
\begin{proof}
    Initialize node features as independent Rademacher random variables, $z \in \{-1,1\}^n$. 
    By Lemma~\ref{lem:hutchinson_diagonal_estimator}, $\E[z \odot (A^k z)] = \diag(A^k)$.
    It is well known that $(A^k)_{vv}$ counts the number of closed walks of length $k$ starting and ending at node $v$ (see Lemma~\ref{lem:walk_counting}).
\end{proof}

\subsection{DIAL: Diagonal Access Layer}

Building on Lemma~\ref{lem:hutchinson_diagonal_estimator}, we propose the Diagonal Access Layer (DIAL), a stochastic message-passing module that gives nodes access to structural quantities through diagonal entries of graph-derived operators.
DIAL is based on propagating random features, and correlating the propagated probes with the original random values.
The technique is general and can be applied to directed and undirected graphs for any GNN architecture. 

\paragraph{Random probes}
For initial features, we sample a probe vector $\mathbf z \in \{-1,1\}^n$ with i.i.d.\ Rademacher entries.

\paragraph{Probe propagation}
Let $\mathbf p^{(0)} = \mathbf z$.
DIAL's main hyperparameter is $K$, the number of layers from which diagonal estimates are formed.
For $k=1,\dots,K-1$, DIAL propagates the probe features by
\[
    \mathbf p^{(k)} = \mathcal T_{\mathbf S}^{(k)}(\mathbf p^{(k-1)}),
\]
where $\mathcal T_{\mathbf S}^{(k)}$ is a local propagation operator built from the GSO $\mathbf S$.
For example, choosing $\mathcal T_{\mathbf S}^{(k)}(\mathbf p)=\mathbf S\mathbf p$ gives $\mathbf p^{(k)}=\mathbf S^k \mathbf z$.
	More generally, the stack of probe-propagation layers may implement powers of the GSO or polynomial filters, such as those arising from the linear propagation part of Equation~(\ref{eq:generalized_GNN}). Thus each propagated probe state can be written as $\mathbf p^{(k)}=M_k \mathbf z$ for some graph-derived operator~$M_k$.

\paragraph{Diagonal readout}
For each state $k=0,\dots,K-1$, DIAL forms a local correlation with the original probe,
\[
    \mathbf d^{(k)} = \mathbf z \odot \mathbf p^{(k)},
\]
where $\odot$ denotes the elementwise product.
For every $k$, Lemma~\ref{lem:hutchinson_diagonal_estimator} gives
\[
    \E\!\left[\mathbf z \odot \mathbf p^{(k)}\right]
    =
    \diag(M_k).
\]
Thus the $K$ DIAL readouts expose structural signals corresponding to the diagonals of $M_0,\dots,M_{K-1}$. For instance, if $M_k=\mathbf A^k$, these signals give closed-walk counts up to length $K-1$.

\paragraph{Aggregation}
In practice, we reduce variance by using $r$ independent probes $\mathbf z^{(1)}, \dots, \mathbf z^{(r)}$.
Equivalently, stacking them as $\mathbf Z = [\mathbf z^{(1)}, \dots, \mathbf z^{(r)}] \in \RR^{n \times r}$ introduces a sample dimension, and the same probe propagation is applied column-wise to produce estimates $\mathbf D^{(k)} \in \RR^{n \times r}$.
The state-wise estimates are then aggregated across propagation depths and independent probes using a shared readout:
\[
    \mathbf h_i^{\mathrm{DIAL}}
    =
    \frac{1}{r}\sum_{s=1}^r
    \MLP\!\left(
        \big\|_{k=0}^{K-1} \mathbf D^{(k)}_{i,s}
    \right)
    \in \RR^{d_{\mathrm{dial}}}.
\]
Averaging over $r$ probes reduces the estimator variance at a rate of $1/r$, while the distribution of the probes makes the construction permutation equivariant in expectation.

The resulting DIAL layer can be stacked with a standard message-passing GNN by using $\mathbf h_i^{\mathrm{DIAL}}$ directly as a structural node representation or concatenating it with ordinary node features,
\[
    \mathbf H^{(0)}_i = \mathbf x_i \mathbin{\|} \mathbf h_i^{\mathrm{DIAL}}.
\]

\section{Structural Embeddings via Diagonal Estimation}
\label{sec:SEs}

Many classical node-level structural features, such as closed-walk counts and centrality measures, can be expressed as node-wise statistics derived from linear operators on the graph.
In this section, we first show that DIAL gives nodes access to diagonal entries of graph filters.
This can be used to recover several classical node quantities relevant for structural embeddings, including closed-walk counts, subgraph centrality, heat kernel signatures and graph angles.
We also extend the diagonal estimation technique to off-diagonals for neighboring nodes, which is useful for counting simple cycles. Proofs for this section are in Appendix~\ref{app:SEs}.

\subsection{Diagonals of Graph Filters}

Recall that a polynomial filter of degree $K-1$ can be expressed as
$$p(S) = \sum_{k=0}^{K-1} a_k S^k$$
where $S$ is a graph shift operator. With diagonalization, this can also be written as $p(S)=\Phi p(\Lambda) \Phi^T$, where $p(\Lambda) = \diag(p(\lambda_1), \dots, p(\lambda_n))$. In particular, spectral filters act independently on each eigenspace of $S$.

\begin{figure}[t]
    \centering
    \includegraphics[width=0.85\linewidth]{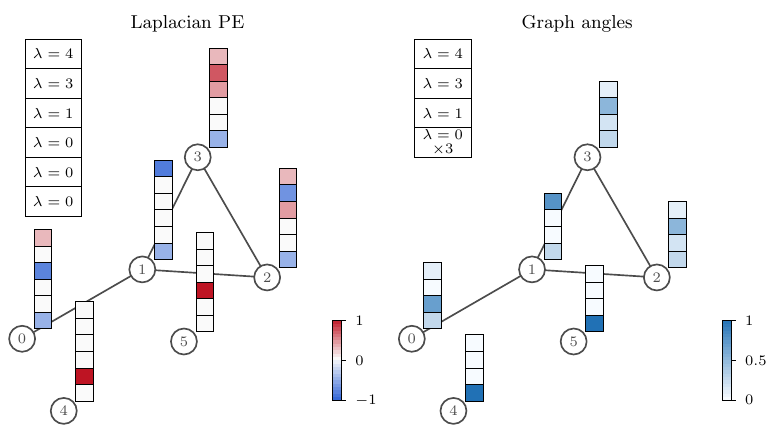}
    \caption{Graph-angle structural encodings compared to Laplacian positional encodings. \textbf{Left}:  Laplacian PE represents node $i$ by $\Phi_{i,:}$, the row of eigenvector coordinates (in practice only a subset). This captures relative position in the graph: nodes 0 and 1 are close to each other so their encodings are similar, whereas nodes 4 and 5 are in different components so their encodings can be very different. \textbf{Right}: graph angles represent node $i$ by $((P_{\mu_j})_{ii})_{j=1}^q$, where $\mu_1,\ldots,\mu_q$ are the distinct (Laplacian) eigenvalues. Graph angles are an example of a structural encoding: unlike eigenvectors, graph angles assign the same representation to automorphic nodes (such as 4 and 5). They are also invariant to basis and sign choices.}
    \label{fig:graph_angle_structural_encoding}
\end{figure}

The linear propagation part of Equation~(\ref{eq:generalized_GNN}) allows such filters to be applied in message passing effectively. However, many relevant structural properties, such as cycle counts, depend on \textit{self-interactions}. Intuitively, this is encoded in the diagonal entries of $p(S)$. The diagonal is not accessible via standard message passing, which aggregates information row-wise and cannot isolate self-interactions. The following theorem shows that diagonal entries can be accessed in expectation using DIAL:

\begin{restatable}{theorem}{diagSpectralFilter}
    \label{thm:diagonal_of_spectral_filter}
    DIAL can compute node-wise statistics of the form
    $$v \mapsto (p(S))_{vv}$$
    in expectation, where $p$ is a polynomial filter of degree at most $K-1$. Equivalently, DIAL can compute expressions of the form
    $$v \mapsto \sum_{i=1}^n g(\lambda_i) (\boldsymbol\phi_i(v))^2$$
    where $\lambda_i$ and $\boldsymbol\phi_i$ are the eigenvalues and eigenvectors of the GSO, and $g$ is any polynomial function of degree at most $K-1$. Moreover, multiple polynomial filters can be expressed in parallel by increasing the (output) feature dimension.
\end{restatable}%

\subsection{Structural Quantities}

The diagonal-filter view gives a compact way to recover several standard structural embeddings.
We first focus on the quantities that can be expressed as polynomials of the GSO.

\paragraph{Example: graph angles.}
Figure~\ref{fig:graph_angle_structural_encoding} illustrates the distinction between positional and structural spectral encodings.
Graph angles, defined in the preliminaries, are a canonical example of diagonal spectral features.
They describe how strongly a node participates in an eigenspace through $\angle_\mu(v)=(P_\mu)_{vv}$.
This differs from eigenvector positional encodings, which assign nodes coordinates in selected eigenvectors.
While eigenvector coordinates are useful positional features, they depend on sign choices and, for repeated eigenvalues, on an arbitrary basis for the eigenspace.
Graph angles instead summarize the whole eigenspace through the projector $P_\mu$, making them intrinsic structural features.

\begin{restatable}{theorem}{structuralQuantities}
    \label{thm:structuralQuantities}
    DIAL can compute the following node-level structural quantities in expectation:
    \begin{itemize}
        \renewcommand{\labelitemi}{$\bullet$}
        \item \textbf{Closed-walk counts.}
        For every $1\le k\le K-1$,
        \[
            v \mapsto (A^k)_{vv},
        \]
        gives the number of closed walks of length $k$ starting and ending at $v$.

        \item \textbf{Graph angles.}
        For a Laplacian eigenspace $E_\mu$ with projector $P_\mu$,
        \[
            v \mapsto \angle_\mu(v)=(P_\mu)_{vv},
        \]
        can be computed, provided the eigenspace projector $P_\mu$ is represented by a polynomial in $L$ of degree at most $K-1$.
        In particular, this holds if $L$ has at most $K$ distinct eigenvalues.
    \end{itemize}
\end{restatable}

\noindent Spectral quantities defined by exponential graph filters can be approximated with polynomials of the GSO:
\begin{restatable}{theorem}{approxStructuralQuantities}
    \label{cor:approx_structural_quantities}
    DIAL can approximate the following node-level structural quantities in expectation:
    \begin{itemize}
        \item \textbf{Subgraph centrality.}
        For $\SC(v)=(e^A)_{vv}$, if $A$ is symmetric and $\|A\|_2\le \alpha$, the degree-$(K-1)$ Taylor approximation yields an estimate $\widehat{\SC}_K(v)$ satisfying
        \[
            |\E[\widehat{\SC}_K(v)]-\SC(v)|
            \le
            \frac{e^\alpha \alpha^K}{K!}.
        \]

        \item \textbf{Heat kernel signature.}
        For $\mathrm{HKS}_t(v)=(e^{-tL})_{vv}$, if $L$ is symmetric and $\|L\|_2\le \beta$, the degree-$(K-1)$ Taylor approximation yields an estimate $\widehat{\mathrm{HKS}}_{t,K}(v)$ satisfying
        \[
            |\E[\widehat{\mathrm{HKS}}_{t,K}(v)]-\mathrm{HKS}_t(v)|
            \le
            \frac{(t\beta)^K}{K!}.
        \]
    \end{itemize}
\end{restatable}

\noindent Estimator variances for both the exact and approximate quantities above are derived in Appendix~\ref{app:estimator_variance}.

\subsection{Neighbor Off-Diagonal Readout}

Some structural quantities require off-diagonal entries of graph-derived operators.
For example, the number of common neighbors of adjacent nodes $u$ and $v$ is $(A^2)_{uv}$, which is useful for link prediction, and also appears in formulas for counting simple cycles.
The same random probing idea can estimate such entries: nodes first share their original probe values with their neighbors, and after $\ell$ layers of message propagation, each endpoint correlates the last layer state with the original probe values of their neighbors.

\begin{restatable}{theorem}{neighborOffDiagonalFilter}
    \label{thm:neighbor_off_diagonal_filter}
    For each node $v$, DIAL can compute the multiset of off-diagonal elements of $p(S)$ over neighbors $w \in N(v)$:
    \[
        \{\!\{(p(S))_{vw}: w \in N(v)\}\!\}
    \]
    in expectation, where $p$ is a polynomial filter of degree at most $K-1$.
\end{restatable}
\noindent The proof is in Appendix~\ref{app:SEs}.

Combined with closed-walk counts from Theorem~\ref{thm:structuralQuantities}, the off-diagonal estimation from Theorem~\ref{thm:neighbor_off_diagonal_filter} can be used to express counts of simple cycles:

\begin{theorem}
    \label{lem:cycle_counting}
    A DIAL-augmented message-passing model can compute the number of simple cycles of lengths 3,4 and 5 at each node.
\end{theorem}
The proof of Theorem~\ref{lem:cycle_counting} is given in Section~\ref{app:cycle_counting_proofs}.

\section{Experiments}

We present empirical results in three main settings: molecular property prediction on ZINC \cite{irwin2012zinc, dwivedi2023benchmarking}, synthetic node-level cycle counting \cite{huang2022boosting}, and directed structural pattern detection on synthetic multigraphs \cite{egressy2024provably}.
These experiments test the usefulness of DIAL as a lightweight structural encoding, and its ability to recover explicitly structural information in controlled settings. 

Our comparisons include standard message-passing GNNs, hand-crafted cycle features, spectral positional encodings, PEARL and other random-probe baselines, as well as subgraph and higher-order architectures for cycle counting.
Random Feature Propagation (RFP) \cite{eliasof2023graph} is a closely related baseline, but we omit it from the experiments because no public implementation was available.\footnote{Since no public implementation was available, we implemented RFP based on the paper. However, the paper reports the hyperparameter search space rather than the selected settings, and our implementation did not match the reported ZINC performance closely enough to support a reliable comparison.}

\subsection{Setup}

We use the PEARL codebase \citep{kanatsoulis2025learning} as the starting point, modifying the positional-encoding module for DIAL. In the DIAL variants reported here, Hutchinson correlations are applied at two levels inside the structural encoding module: first after explicit graph-shift propagation of the random probes, and then again after a number of learned GIN layers. The outputs of each layer are concatenated and projected to the final node encoding. Full experimental details can be found in Appendix~\ref{app:experimental_details}.

\subsection{Results}

Table~\ref{tab:zinc_logp} reports a PE-focused comparison on ZINC, with logP prediction evaluated by test MAE.
The non-DIAL entries follow the baseline set reported by \citet{kanatsoulis2025learning}. Hence, the table should be read as a comparison to positional-encoding methods rather than as a comprehensive ZINC leaderboard.
DIAL performs comparably to strong spectral and random-probe positional encoding methods.

\begin{table}[t]
    \caption{Test MAE for logP prediction on ZINC; lower is better. For SignNet, BasisNet, and SPE, ``8 smallest'' uses the eigenvectors corresponding to the eight smallest eigenvalues of the normalized Laplacian, whereas ``full'' uses the complete eigenbasis. Best within each block is bolded.}
    \label{tab:zinc_logp}
    \centering
      \begin{tabular}{lc}
        \toprule
        \textbf{Method} & \textbf{Test MAE} \\
        \midrule
        \text{GCN} \cite{kipf2016semi}  & $0.469 \pm 0.002$ \\
        \text{GIN} \cite{xu2018powerful} & $0.209 \pm 0.018$  \\
        \text{GSN with cycles} \cite{bouritsas2022improving} & $\boldsymbol{0.115 \pm 0.012}$ \\
        \text{GIN + random ID} \cite{xu2018powerful} & $0.279 \pm 0.023$ \\
    
    \midrule
        
        SignNet (8 smallest) \cite{lim2022sign} & $0.1034 \pm 0.0056$ \\
        SignNet (full) \cite{lim2022sign} & $0.0853 \pm 0.0026$ \\
        BasisNet (8 smallest) \cite{lim2022sign} & $0.1554 \pm 0.0048$ \\
        BasisNet (full) \cite{lim2022sign} & $0.1555 \pm 0.0124$ \\
        SPE (8 smallest) \cite{huang2024stability} & $0.0736 \pm 0.0007$ \\
        SPE (full) \cite{huang2024stability} & $0.0693 \pm 0.0040$ \\
        R-PEARL \cite{kanatsoulis2025learning} & $0.0699 \pm 0.002$ \\
        B-PEARL \cite{kanatsoulis2025learning} & $\boldsymbol{0.0644 \pm 0.001}$ \\
        \midrule
        DIAL ($K=1$)  & $0.0703 \pm 0.0012$ \\
        DIAL ($K=4$) & $\boldsymbol{0.0658 \pm 0.0020}$ \\
        DIAL ($K=8$) & $0.0733 \pm 0.0026$ \\
        DIAL ($K=12$) & $0.0752 \pm 0.0074$ \\
    
        \bottomrule
    \end{tabular}
\end{table}

\begin{table}[!htbp]
    \caption{Normalized MAE for node-level cycle counting. Complexities assume fixed hidden width; $N$ is the number of nodes, $\bar d$ the average degree, $s$ the maximum extracted subgraph size, $r$ the number of probes, $L$ the number of learned layers or blocks, and $K$ the propagation depth. Inference space excludes graph storage. The lower block contains the random-probe methods; best results are bolded separately in the final two blocks. R-PEARL and DIAL are reported as mean $\pm$ standard deviation.}
    \label{tab:cycle_counting}
    \vspace{4pt}
    \centering
    \setlength{\tabcolsep}{2pt}
    
    \begin{tabular}{lcccccc}
    \toprule
    \multirow{2}{*}{Method} & \multicolumn{2}{c}{Complexity} & \multicolumn{4}{c}{Synthetic (norm. MAE)} \\ 
    \cmidrule(r){2-3} \cmidrule(r){4-7}
    & Space & Time & 3-Cycle     & 4-Cycle  & 5-Cycle & 6-Cycle  \\ 
    \midrule
    Base GNN  & $O(N)$ & $O(LN\bar d)$ & 0.3515 & 0.2742   &  0.2088   & 0.1555 \\
    \midrule 
    ID-GNN \cite{you2021identity} & $O(Ns)$ & $O(LNs\bar d)$ & 0.0006 & 0.0022 & 0.0490 & 0.0495 \\
    NGNN \cite{zhang2021nested} & $O(Ns)$ & $O(LNs\bar d)$ & \textbf{0.0003}  & 0.0013 & 0.0402 & 0.0439 \\
    GNN-AK+ \cite{zhao2021stars}  & $O(Ns)$ & $O(LNs\bar d)$ & 0.0004 & 0.0041  & 0.0133  & 0.0238 \\
    PPGN \cite{maron2019provably}  & $O(N^2)$ & $O(LN^3)$ & \textbf{0.0003}    & \textbf{0.0009}  & 0.0036 & \textbf{0.0071} \\
    I$^2$-GNN \cite{huang2022boosting}  & $O(Ns\bar d)$ & $O(LNs\bar d^2)$ & \textbf{0.0003}  & 0.0016   & \textbf{0.0028} & 0.0082 \\
    2-DRFWL(2) \cite{zhou2023distance} & $O(N\bar d^2)$ & $O(LN\bar d^4)$ & 0.0004 & 0.0015 & 0.0034 & 0.0087 \\
    \midrule
    R-PEARL \cite{kanatsoulis2025learning} & $O(rN)$ & $O((K+L)rN\bar d)$
    & {\tiny $.0118\!\pm\!.0037$}
    & {\tiny $.0360\!\pm\!.0054$}
    & {\tiny $.0561\!\pm\!.0035$}
    & {\tiny $.0622\!\pm\!.0025$} \\
    DIAL (Ours) & $O(KrN)$ & $O((K+L)rN\bar d)$
    & {\tiny $\boldsymbol{.0041\!\pm\!.0001}$}
    & {\tiny $\boldsymbol{.0185\!\pm\!.0027}$}
    & {\tiny $\boldsymbol{.0271\!\pm\!.0005}$}
    & {\tiny $\boldsymbol{.0348\!\pm\!.0030}$} \\
    \bottomrule
    \end{tabular}
\end{table}

Table~\ref{tab:cycle_counting} compares DIAL with several baselines on node-level cycle counting. Methods designed specifically to expose cycle information, such as subgraph and higher-order architectures, achieve the lowest errors, but require substantially higher space or time complexity. Among random-probe methods, DIAL improves over R-PEARL across all cycle lengths while retaining linear scaling in $N$, $\bar d$, and $r$.

We further evaluate DIAL on a directed synthetic pattern-detection task introduced by \citet{egressy2024provably}. See Table~\ref{tab:directed_pattern_detection} for the results. The task is node-level multi-label classification: nodes are labeled by whether they participate in directed cycles of length $2$--$6$, scatter-gather motifs (S-G), or bicliques (B-C), following \citet{altman2023realistic}. DIAL outperforms R-PEARL on every directed pattern. The improvement is especially large for short directed cycles. Overall, DIAL raises macro F1 from 52.24 to 77.84. This suggests that correlating propagated representations with the original probes provides a substantially stronger signal for directed structural patterns.

\begin{table*}[!htbp]
    \caption{Pattern detection on synthetic directed graphs. We report minority-class F1 scores in percent for node-level binary labels; higher is better. The detected patterns include directed cycles of length $2$--$6$, scatter-gather motifs (S-G), and bicliques (B-C). Both methods use $K=4$ and six learned probe-aggregation layers. Best in each column is bolded.}
    \label{tab:directed_pattern_detection}
    \centering
    \resizebox{\textwidth}{!}{%
    \begin{tabular}{lcccccccc}
        \toprule
        \textbf{Method} %
        & \textbf{2-Cycle} & \textbf{3-Cycle} & \textbf{4-Cycle} & \textbf{5-Cycle} & \textbf{6-Cycle} & \textbf{S-G} & \textbf{B-C} & \textbf{Macro F1} \\
        \midrule
        R-PEARL \cite{kanatsoulis2025learning} %
        & $15.48 \pm 7.59$ & $41.21 \pm 5.42$ & $66.05 \pm 0.52$ & $55.13 \pm 0.80$ & $54.74 \pm 0.60$ & $70.54 \pm 1.09$ & $62.56 \pm 2.49$ & $52.24$ \\
        DIAL (Ours) %
        & $\boldsymbol{91.07 \pm 3.33}$ & $\boldsymbol{82.44 \pm 4.48}$ & $\boldsymbol{79.72 \pm 1.50}$ & $\boldsymbol{72.10 \pm 2.21}$ & $\boldsymbol{74.83 \pm 2.95}$ & $\boldsymbol{78.68 \pm 2.79}$ & $\boldsymbol{66.07 \pm 1.48}$ & $\boldsymbol{77.84}$ \\
        \bottomrule 
    \end{tabular}
    }
\end{table*}

\clearpage
\section{Conclusions}

We introduce DIAL, a message-passing layer that uses randomized probing to give nodes access to structural self-interactions expressed as the diagonals of graph-derived operators. We showed that DIAL can express or approximate closed-walk and short-cycle counts, graph angles, subgraph centrality, and heat kernel signatures. In our experiments, DIAL is competitive with strong positional-encoding methods on ZINC. It also outperforms comparable random-probe methods on undirected cycle counting and directed pattern detection. Overall, DIAL provides a simple way to improve structural expressivity without using higher-order architectures or a fixed set of motifs.

\section*{Acknowledgements}

The support of the Swiss National Science Foundation (project number: 212158) for this work is gratefully acknowledged.

\bibliographystyle{unsrtnat}
\bibliography{ref}

\appendix

\section{Proofs for Section~\ref{sec:SEs}}
\label{app:SEs}

\diagSpectralFilter*
\begin{proof}
    Let $p(S)=\sum_{k=0}^{K-1} a_k S^k$ be the target polynomial filter.
    Since the linear propagation part of Equation~(\ref{eq:generalized_GNN}) represents polynomial graph filters of degree $K-1$, the DIAL probe-propagation weights can learn the coefficients of $p$.
    Thus, for a Rademacher probe $\mathbf z$, the learned probe propagation can produce
    \[
        \mathbf p = p(S)\mathbf z.
    \]
    The DIAL readout then forms the node-wise correlation
    \[
        \mathbf z \odot \mathbf p
        =
        \mathbf z \odot p(S)\mathbf z.
    \]
    Applying Lemma~\ref{lem:hutchinson_diagonal_estimator} with $M=p(S)$ gives
    \[
        \E\!\left[\mathbf z \odot p(S)\mathbf z\right]
        =
        \diag(p(S)).
    \]
    Hence node $v$ obtains $(p(S))_{vv}$ in expectation.

    For the equivalent spectral expression, write $S=\Phi \Lambda \Phi^T$.
    Then $p(S)=\Phi p(\Lambda)\Phi^T$, and therefore
    \[
        (p(S))_{vv}
        =
        \sum_{i=1}^n p(\lambda_i)(\boldsymbol\phi_i(v))^2.
    \]
    Since $p$ is an arbitrary polynomial of degree at most $K-1$, this gives the stated expression for any polynomial response $g$ of that degree.
    Multiple polynomial filters can be handled in parallel by using separate output channels for $p_1,\dots,p_d$ and applying the same argument channel-wise.
\end{proof}

\neighborOffDiagonalFilter*
\begin{proof}
    Let $M=p(S)$.
    Since the linear propagation part of Equation~(\ref{eq:generalized_GNN}) represents polynomial graph filters of degree $K-1$, the DIAL probe-propagation weights can learn the coefficients of $p$.
    Thus, after the corresponding probe-propagation layer, the probe feature is
    \[
        \mathbf p^{(\ell)} = M\mathbf z.
    \]
    For the readout, each node keeps its original probe value $z$ alongside the propagated state $\mathbf p^{(\ell)}=M\mathbf z$.
    A subsequent GIN layer allows node $v$ to receive the multiset of probe values $z_w$ of its neighbors while retaining $p^{(\ell)}_v$.
    Node $v$ can therefore form $z_w p^{(\ell)}_v$ for every $w\in N(v)$, and permutation-invariant aggregation encodes the resulting unordered collection
    \[
        R_v^{(\ell)}
        =
        \{\!\{z_w p^{(\ell)}_v : w\in N(v)\}\!\}.
    \]
    To analyze this, temporarily retain the identity of the sending neighbor and write $r_{vw}^{(\ell)}=z_w p^{(\ell)}_v$.
    Since $p^{(\ell)}_v=\sum_j M_{vj}z_j$, we have
    \[
        \E[r_{vw}^{(\ell)}]
        =
        \E\left[\sum_j M_{vj} z_w z_j\right]
        =
        \sum_j M_{vj}\E[z_w z_j]
        =
        M_{vw},
    \]
    because $\E[z_w z_j]=0$ for $j \neq w$ and $\E[z_w^2]=1$.
    Therefore $R_v^{(\ell)}$ estimates the multiset $\{\!\{M_{vw}: w\in N(v)\}\!\}$ in expectation.
    Substituting $M=p(S)$ gives the claim.
\end{proof}

\subsection{Estimator Variance and Concentration}
\label{app:estimator_variance}

Lemma~\ref{lem:hutchinson_diagonal_estimator} establishes unbiasedness for diagonal entries, and Theorem~\ref{thm:neighbor_off_diagonal_filter} extends this to neighbor off-diagonal entries. Here, we compute the variance of these estimators in terms of the number of probes.

\begin{lemma}
    \label{lem:matrix_entry_variance}
    Let $M\in\RR^{n\times n}$ be fixed, and let $\mathbf z^{(1)},\ldots,\mathbf z^{(r)}$ be independent Rademacher probes. For fixed nodes $v,w$, define
    \[
        \widehat M_{vw}
        =
        \frac{1}{r}\sum_{s=1}^r
        z_w^{(s)}(M\mathbf z^{(s)})_v.
    \]
    The variance of this estimator is
    \begin{equation}
        \label{eq:estimator_variance}
        \Var(\widehat M_{vw})
        =
        \frac{1}{r}\sum_{j\neq w}M_{vj}^2
        =
        \frac{\lVert M_{v,:}\rVert_2^2-M_{vw}^2}{r}.
    \end{equation}
    If $M$ is symmetric, the diagonal case simplifies to
    \[
        \Var(\widehat M_{vv})
        =
        \frac{(M^2)_{vv}-M_{vv}^2}{r}.
    \]
\end{lemma}

\begin{proof}
    For one probe, write $X_{vw}=z_w(M\mathbf z)_v$ and separate the desired entry from the remaining terms:
    \begin{align*}
        X_{vw}
        &=
        z_w\sum_j M_{vj}z_j \\
        &=
        M_{vw}z_w^2+\sum_{j\neq w}M_{vj}z_wz_j \\
        &=
        M_{vw}+\sum_{j\neq w}M_{vj}z_wz_j.
    \end{align*}
    Thus, the desired entry $M_{vw}$ is recovered as a deterministic term and does not contribute to the variance. Every other entry in row $v$ contributes a zero-mean random term. Expanding the square of these remaining terms gives
    \begin{align*}
        \Var(X_{vw})
        &=
        \E\Big[\Big(\sum_{j\neq w}M_{vj}z_wz_j\Big)
                   \Big(\sum_{k\neq w}M_{vk}z_wz_k\Big)\Big] \\
        &=
        \sum_{j,k\neq w}M_{vj}M_{vk}\E[z_jz_k]
        =
        \sum_{j\neq w}M_{vj}^2,
    \end{align*}
    since the cross terms vanish. Finally, independence across probes gives
    \[
        \Var(\widehat M_{vw})
        =
        \frac{1}{r^2}\sum_{s=1}^r\Var(X_{vw}^{(s)})
        =
        \frac{1}{r}\sum_{j\neq w}M_{vj}^2.
    \]
    The row-norm expression follows directly. If $M$ is symmetric, then $\lVert M_{v,:}\rVert_2^2=(M^2)_{vv}$.
\end{proof}

\begin{corollary}
    \label{cor:structural_estimator_variances}
    For $r$ independent probes, the random-probe estimators satisfy the following:
    \begin{itemize}
        \item \textbf{Closed walks and triangles.}
        On an undirected graph,
        \[
            \Var\!\left(\widehat{(A^k)_{vv}}\right)
            =
            \frac{(A^{2k})_{vv}-(A^k)_{vv}^2}{r}.
        \]
        In particular, for a simple graph without self-loops, $C_3(v)=(A^3)_{vv}/2$, so
        \[
            \Var\!\left(\widehat C_3(v)\right)
            =
            \frac{(A^6)_{vv}-(A^3)_{vv}^2}{4r}.
        \]

        \item \textbf{Graph angles.}
        Let $P_\mu$ be the orthogonal projector onto the eigenspace associated with $\mu$, and assume it is represented by a polynomial filter available to DIAL, as in Theorem~\ref{thm:structuralQuantities}. Then
        \[
            \Var\!\left(\widehat{(P_\mu)_{vv}}\right)
            =
            \frac{\angle_\mu(v)(1-\angle_\mu(v))}{r}
            \leq
            \frac{1}{4r}.
        \]

        \item \textbf{Subgraph centrality and heat kernel signature.}
        Define the polynomial filters
        \[
            p_{K-1}(A)=\sum_{\ell=0}^{K-1}\frac{A^\ell}{\ell!},
            \qquad
            h_{t,K-1}(L)=\sum_{\ell=0}^{K-1}\frac{(-tL)^\ell}{\ell!}.
        \]
        Their estimator variances are
        \begin{align*}
            \Var\!\left(\widehat{\SC}_K(v)\right)
            &=
            \frac{(p_{K-1}(A)^2)_{vv}-(p_{K-1}(A)_{vv})^2}{r}, \\
            \Var\!\left(\widehat{\mathrm{HKS}}_{t,K}(v)\right)
            &=
            \frac{(h_{t,K-1}(L)^2)_{vv}-(h_{t,K-1}(L)_{vv})^2}{r}.
        \end{align*}
    \end{itemize}
\end{corollary}

\begin{proof}
    For closed walks, set $M=A^k$ in Lemma~\ref{lem:matrix_entry_variance} and use $(A^k)^2=A^{2k}$. Scaling the resulting estimate by $1/2$ gives the triangle-count variance. For graph angles, set $M=P_\mu$ and use $P_\mu^2=P_\mu$. Finally, substituting $M=p_{K-1}(A)$ and $M=h_{t,K-1}(L)$ gives the subgraph-centrality and heat-kernel formulas.
\end{proof}

\paragraph{Combining estimates.}
Lemma~\ref{lem:matrix_entry_variance} also gives the variance of each randomized matrix-entry estimate used in the longer-cycle constructions. When several estimates formed from the same probes are combined (as in Theorem~\ref{lem:cycle_counting}), the variance of the final quantity additionally depends on their covariance.

\paragraph{Concentration.}
Let $\sigma_{vw}^2=\sum_{j\neq w}M_{vj}^2$. The products $\{z_wz_j:j\neq w\}$ are jointly independent Rademacher variables, so for every $\epsilon>0$,
\[
    \pr\!\left(\left|\widehat M_{vw}-M_{vw}\right|\geq\epsilon\right)
    \leq
    2\exp\!\left(-\frac{r\epsilon^2}{2\sigma_{vw}^2}\right).
\]

\subsection{Structural Quantities}

The next proofs use a Lagrange polynomial, which is a construction for polynomials of minimal degree interpolating a given set of points:

\begin{lemma}[Lagrange interpolation polynomials]
    \label{lem:Lagrange_interp}
    Given a set of $k+1$ points $\{(x_1, y_1), \dots, (x_{k+1}, y_{k+1})\}$, where $x_i \neq x_j$, there is a polynomial $L(x)$ of degree $\le k$ such that $L(x_i) = y_i$ for all $1 \le i \le k+1$.
\end{lemma}
\begin{proof}
    For each $i$, construct a polynomial $\ell_i(x)$ such that $\ell_i(x_i)=1$ and $\ell_i(x_j)=0$ for all $j \neq i$. Such a polynomial is given by
    $$
        \ell_i(x)
        = \prod_{j \neq i: 1 \le j \le k+1} \frac{x-x_j}{x_i-x_j}
    $$
    since all factors are equal to one for $x=x_i$, and for $x=x_j, j \neq i$, $(x-x_j)/(x_i-x_j)=0$. Note that the degree of $\ell_i(x)$ is $k$. The polynomial $L(x)$ is given by $\sum_{i=1}^{k+1} y_i \ell_i(x)$.
\end{proof}

\structuralQuantities*
\begin{proof}
    \textbf{Closed walks:} With the adjacency matrix $A$ as the GSO, $(A^k)_{vv}$ gives the number of closed walks of length $k$ starting and ending at $v$ (Lemma~\ref{lem:walk_counting}). This is the diagonal of the polynomial filter $p(A)=A^k$, which has degree $k$. Theorem~\ref{thm:diagonal_of_spectral_filter} gives access to $(A^k)_{vv}$ in expectation.

    \textbf{Graph angles:} Use the Laplacian $L$ as the GSO, and let $\eta_1,\dots,\eta_m$ be the distinct eigenvalues of $L$.
    Fix an eigenvalue $\mu\in\{\eta_1,\dots,\eta_m\}$.
    Let $P_\mu=\Phi_\mu\Phi_\mu^T$ be the projector onto the eigenspace $E_\mu$, so that the graph angle is $\angle_\mu(v)=(P_\mu)_{vv}$.
    By Lemma~\ref{lem:Lagrange_interp}, there is a polynomial $p_\mu$ of degree at most $m-1$ such that
    $$
        p_\mu(\eta) = \begin{cases}
            1 & \eta = \mu \\
            0 & \eta \in \{\eta_1,\dots,\eta_m\}\setminus\{\mu\}.
        \end{cases}
    $$
    With this choice of polynomial,
    \[
        p_\mu(L)
        =
        \Phi p_\mu(\Lambda)\Phi^T
        =
        \sum_{i\in I_\mu}\boldsymbol\phi_i\boldsymbol\phi_i^T
        =
        P_\mu.
    \]
    If the Laplacian has at most $K$ distinct eigenvalues, then $\deg(p_\mu)\le m-1\le K-1$.
    Theorem~\ref{thm:diagonal_of_spectral_filter} then gives access to $\diag(P_\mu)$ in expectation, which is exactly the vector of graph angles for $E_\mu$.
\end{proof}
Note that for graph angles, this theorem is an exact-recovery result. Since an $n$-node graph may have up to $n$ distinct Laplacian eigenvalues, the interpolation construction may require a large $K$. 

\paragraph{Subgraph centrality and heat kernel signature.}
For completeness, the subgraph centrality of node $v$ is
\[
    \SC(v)=(e^A)_{vv}
    =\sum_{k=0}^{\infty}\frac{(A^k)_{vv}}{k!},
\]
and thus measures participation in closed walks of all lengths, with longer walks receiving smaller weights.
Similarly, if $L=\Phi\Lambda\Phi^T$, the heat kernel signature is
\[
    \mathrm{HKS}_t(v)=(e^{-tL})_{vv}
    =\sum_{k=0}^{\infty}\frac{(-t)^k}{k!}(L^k)_{vv}
    =\sum_i e^{-t\lambda_i}(\boldsymbol\phi_i(v))^2.
\]
It gives the heat at $v$ at time $t$ when one unit of heat is initially placed at $v$.

\approxStructuralQuantities*
\begin{proof}
    For subgraph centrality, use the degree-$(K-1)$ Taylor polynomial
    \[
        p_{K-1}(x)=\sum_{r=0}^{K-1}\frac{x^r}{r!}.
    \]
    By Theorem~\ref{thm:diagonal_of_spectral_filter}, DIAL can compute an estimate $\widehat{\SC}_K(v)$ such that
    \[
        \E[\widehat{\SC}_K(v)] = (p_{K-1}(A))_{vv}.
    \]
    Since $A$ is symmetric, write $A=\Phi\Lambda\Phi^T$.
    Then
    \[
        \|e^A-p_{K-1}(A)\|_2
        =
        \max_i |e^{\lambda_i}-p_{K-1}(\lambda_i)|.
    \]
    Taylor's theorem gives
    \[
        |e^\lambda-p_{K-1}(\lambda)|
        \le
        \frac{e^\alpha \alpha^K}{K!}
        \qquad \text{for all } |\lambda|\le \alpha.
    \]
    Therefore
    \[
        |\E[\widehat{\SC}_K(v)]-\SC(v)|
        =
        |(p_{K-1}(A)-e^A)_{vv}|
        \le
        \|p_{K-1}(A)-e^A\|_2
        \le
        \frac{e^\alpha \alpha^K}{K!}.
    \]

    For the heat kernel signature at time $t$, use
    \[
        h_{t,K-1}(x)=\sum_{r=0}^{K-1}\frac{(-tx)^r}{r!}.
    \]
    Theorem~\ref{thm:diagonal_of_spectral_filter} gives an estimate $\widehat{\mathrm{HKS}}_{t,K}(v)$ satisfying
    \[
        \E[\widehat{\mathrm{HKS}}_{t,K}(v)]
        =
        (h_{t,K-1}(L))_{vv}.
    \]
    Since $L$ is symmetric positive semidefinite, its eigenvalues satisfy $0\le \lambda_i\le \beta$.
    Taylor's theorem applied to $e^x$ on $x\in[-t\beta,0]$ gives
    \[
        |e^{-t\lambda}-h_{t,K-1}(\lambda)|
        \le
        \frac{(t\beta)^K}{K!}
        \qquad \text{for all } 0\le \lambda\le \beta.
    \]
    As above,
    \[
        |\E[\widehat{\mathrm{HKS}}_{t,K}(v)]-\mathrm{HKS}_t(v)|
        \le
        \|h_{t,K-1}(L)-e^{-tL}\|_2
        \le
        \frac{(t\beta)^K}{K!}.
    \]
\end{proof}

\subsection{Simulating BasisNet}

BasisNet \cite{lim2022sign} computes functions of the projection matrices $P_{\mu_i} = \Phi_{\mu_i} \Phi_{\mu_i}^T$. Permutation equivariance requires that any such function $\varphi$ satisfy
$$\varphi(\Pi \Phi\Phi^T \Pi^T) = \Pi \varphi(\Phi\Phi^T)$$
for any permutation matrix $\Pi$, since a permutation of the nodes permutes both the rows and columns of the projection matrix. The class of functions with this property is characterized in \cite{pmlr-v97-maron19a}, and BasisNet instantiates $\varphi$ using the Invariant Graph Network (IGN) \cite{maron2018invariant}.

IGNs are a family of GNNs that operate on tensors of order $k$ representing $k$-tuples of nodes. \citet{maron2018invariant} characterize all linear equivariant layers between such tensors. In the setting of BasisNet, only a \textit{mixed-order equivariant layer} is needed, mapping a matrix $X \in \RR^{n \times n}$ to node features $Y \in \RR^{n \times d}$. The following operations can be shown to span the space of all linear layers of this type.

\begin{itemize}
    \item The diagonal operation: $Y_i = X_{ii}$
    \item Row-wise sum: $Y_i = \sum_{j=1}^n X_{ij}$
    \item Column-wise sum: $Y_i = \sum_{j=1}^n X_{ji}$
    \item Sum of all entries: $Y_i = \sum_{j=1}^n \sum_{k=1}^n X_{jk}$
    \item Sum of diagonal entries: $Y_i = \sum_{j=1}^n X_{jj}$
\end{itemize}

\begin{lemma}
    DIAL with a virtual node can simulate BasisNet.
\end{lemma}
\begin{proof}
    From Theorem~\ref{thm:diagonal_of_spectral_filter}, DIAL can compute $\diag(P_{\mu_i})$ for each eigenspace. This gives the diagonal operation. For a symmetric graph shift operator, the row-wise and column-wise sums can be computed with $z = \mathbf 1$ as the initial features. The sum of all entries and sum of diagonal entries can be computed with the help of a virtual node, by summing over the node features obtained from the previous operations.
\end{proof}

The diagonal estimate is key to obtaining the diagonal operation, which is not possible with standard message passing GNNs. The diagonal operation is the most important one, as it is the only one not constant across nodes:

\begin{lemma}
    Let $L$ be the Laplacian matrix and $P_\lambda=\Phi_\lambda \Phi_\lambda^T$ be the projection matrix to the eigenspace of eigenvalue $\lambda$.
    Then, the diagonal operation is the only non-constant operation among the five operations above when applied to $P_\lambda$.
\end{lemma}
\begin{proof}
    The vector $\mathbf 1$ is an eigenvector with eigenvalue 0, since $L \mathbf 1 = 0$. The row-wise sums are zero for $\lambda \neq 0$: $P_\lambda \mathbf 1 = 0$ since $\mathbf 1$ is orthogonal to these eigenspaces for any $\lambda \neq 0$. The same holds for the column sums by symmetry. For the case $\lambda =0$, $P_0 \mathbf 1 = \mathbf 1$ since $\mathbf 1$ is an eigenvector. Thus, the row-wise and column-wise sums are constant across nodes in all cases.

    The diagonal sum and total sum are constant across nodes by definition. Note that the diagonal sum $\sum_{i} (P_\lambda)_{ii}$ is equal to the multiplicity of $\lambda$, since $\trace(P_\lambda) = \trace(P_\lambda^T) = \trace(\Phi_\lambda^T \Phi_\lambda) = \text{rank}(\Phi_\lambda)$.
\end{proof}

\section{Cycle Counting Proofs}
\label{app:cycle_counting_proofs}

The following lemma is a basic result in graph theory, which we include for completeness.
\begin{lemma}
    \label{lem:walk_counting}
    Let $A$ be the adjacency matrix. $(A^k)_{vw}$ gives the number of walks of length exactly $k$ from $v$ to $w$.
\end{lemma}
\begin{proof}
    By induction on $k$. The base case $k=1$ is true by definition of the adjacency matrix. Assume the claim holds for $k$. The number of walks of length $k+1$ from $v$ to $w$ can be partitioned according to the node $u$ visited at step $k$, where $u$ must be a neighbor of $w$. This is given by $A^k A$:  
    $$(A^{k+1})_{vw} = \sum_{u} (A^k)_{vu} A_{uw}$$
    This completes the induction.
\end{proof}

Let $C_k(v)$ denote the number of distinct simple undirected $k$-cycles containing vertex $v$. $C_k(v) = 0$ for $k \le 2$. Let $\Cd{k}{v}$ denote the number of distinct simple directed $k$-cycles containing vertex $v$. Note that $\Cd{k}{v} = 2 C_k(v)$ since each undirected cycle can be traversed in two directions.

\begin{lemma}[Triangles]
    \label{lem:triangles}
    For any vertex $v$, $\Cd{3}{v} = (A^3)_{vv}$
\end{lemma}
\begin{proof}
    Since there are no self-loops, any closed walk of length 3 is a simple directed 3-cycle.
\end{proof}

\begin{lemma}[4-cycles]
$\Cd{4}{v} = (A^4)_{vv} - d(v)^2 + d(v) - \sum_{w \in N(v)} d(w) $
\end{lemma}
\begin{proof}
    We count the number of closed walks of length 4 that are not simple cycles, starting from $v$. Closed walks of length 4 which are not simple cycles must revisit some node. This node can be $v$, meaning the walk must be of the form $v-x-v-y-v$, which can be formed in $d(v)^2 = ((A^2)_{vv})^2$ ways. If the revisited node is not $v$, the walk must be of the form $v-x-y-x-v$, which can be formed in $\sum_{w \in N(v)} d(w)$ ways. These cases overlap for paths of the form $v-x-v-x-v$, of which there are $d(v)=(A^2)_{vv}$ in total. 
\end{proof}

\begin{lemma}[5-cycles]
    \[
    \Cd{5}{v}
    =
    (A^5)_{vv}
    -
    \Big(
    2d(v)(A^3)_{vv}
    +
    2\sum_{u\in N(v)} d(u)(A^2)_{uv}
    +
    \sum_{u\in N(v)} (A^3)_{uu}
    -
    5(A^3)_{vv}
    \Big).
    \]
    \end{lemma}
    
    \begin{proof}
    The entry $(A^5)_{vv}$ counts all closed walks of length $5$ starting and ending
    at $v$. We subtract those that are not simple $5$-cycles.
    
    Every non-simple closed walk of length $5$ in a simple graph consists of a
    directed triangle together with a backtracking walk of the form $x-y-x$. We count such 5-walks in two cases.
    
    First suppose that the triangle contains $v$. Let $T_v$ be the set of undirected triangles containing $v$, 
    $$T_v=\{\{v,a,b\} : \{v,a\},\{a,b\},\{b,v\}\in E\}$$
    Fix an undirected triangle
    $\{v,a,b\}$. There are two orientations. For one orientation, say
    \[
    v \to a \to b \to v,
    \]
    the backtracking walk can be inserted at one of the four positions $v, a, b, v$. Here, the two occurrences of $v$ are distinct, since the 5-walk is rooted at $v$.
    Thus, before correcting overcounting, this gives
    \[
    2\bigl(2d(v)+d(a)+d(b)\bigr)
    \]
    walks for the triangle $\{v,a,b\}$.
    
    However, if the edge on the backtracking walk is one of the triangle edges, the same walk is
    counted twice: once as a backtrack inserted at one endpoint of the edge and once
    as a backtrack inserted at the other endpoint. For each orientation there are
    three such choices. Therefore the number of non-simple closed $5$-walks whose underlying triangle contains $v$ is
    \[
    \sum_{\{v,a,b\} \in T_v}
    \left(
    4d(v)+2d(a)+2d(b)-6
    \right),
    \]

    We now rewrite this expression in matrix form. First, since $d(v)$ is independent of the triangle,
    \[
    \sum_{\{v,a,b\}\in T_v}4d(v)
    =
    4d(v)|T_v|
    =2d(v) (A^3)_{vv}
    \]
    Second, consider the term involving the degrees of the two non-$v$ vertices:
    \[
    \sum_{\{v,a,b\}\in T_v} \bigl(2d(a)+2d(b)\bigr).
    \]
    We rewrite this by summing over neighbors $u\in N(v)$. A neighbor $u$ appears
    as one of the two non-$v$ vertices in exactly as many triangles as there are
    common neighbors of $u$ and $v$. This number is $(A^2)_{uv}$. Thus
    \[
    \sum_{\{v,a,b\}\in T_v} \bigl(2d(a)+2d(b)\bigr)
    =
    2\sum_{u\in N(v)} d(u)(A^2)_{uv}.
    \]

    Finally,
    \[
    \sum_{\{v,a,b\}\in T_v}6
    =
    6|T_v|
    =
    3(A^3)_{vv}.
    \]
    Combining the three identities gives
    \[
    \sum_{\{v,a,b\}\in T_v}
    \left(
    4d(v)+2d(a)+2d(b)-6
    \right)
    =
    2d(v)(A^3)_{vv}
    +
    2\sum_{u\in N(v)} d(u)(A^2)_{uv}
    -
    3(A^3)_{vv}.
    \]
    
    It remains to count the non-simple closed $5$-walks whose triangle does not
    contain $v$. Such a walk must have the form
    \[
    v \to u \to \text{(a directed triangle walk based at } u\text{)} \to v
    \]
    for some $u\in N(v)$. The number of directed triangles through $u$ is
    $(A^3)_{uu}$. Among these, the directed triangles that also contain $v$ are
    exactly those using the edge $\{u,v\}$, and there are $2(A^2)_{uv}$ of them.
    Therefore this case contributes
    \[
    \sum_{u\in N(v)}
    \left(
    (A^3)_{uu}-2(A^2)_{uv}
    \right).
    \]
    The expression can be simplified using the fact that $\sum_{u\in N(v)} (A^2)_{uv} = (A^3)_{vv}$. Hence,
    the second case contributes
    \[
        - 2(A^3)_{vv} + \sum_{u\in N(v)} (A^3)_{uu}.
    \]
    
    Adding the two cases, the number of non-simple closed $5$-walks from $v$ to
    itself is
    \[
    2d(v)(A^3)_{vv}
    +
    2\sum_{u\in N(v)} d(u)(A^2)_{uv}
    +
    \sum_{u\in N(v)} (A^3)_{uu}
    -
    5(A^3)_{vv}.
    \]
    Subtracting these walks from all closed $5$-walks gives the claimed formula.
    \end{proof}

\section{Cycle Counting in \cite{kanatsoulis2024counting} and \cite{kanatsoulis2025learning}}
\label{app:kanatsoulis}

\citet{kanatsoulis2024counting} study the substructure-counting power of GNNs under random inputs. However, a key assumption in their analysis does not appear to be satisfiable, which undermines the validity of the resulting guarantees. We highlight this issue because the same assumption is directly reused in PEARL \cite{kanatsoulis2025learning}, which is closely related to our method.

In more detail, the analysis in \cite{kanatsoulis2024counting} uses higher-order moments of random variables. A core assumption is access to random variables with all moments equal to 1:
\begin{quote}
    We assume that they $(x_1, \dots, x_N)$ are identically distributed and satisfy $\E[x_i] = 0, \E[x_i^p] = 1$ for $i =1, \dots, N$ and $p \ge 2 \in \mathbb{Z}$
\end{quote}

\noindent While this assumption can hold for even moments (for example by taking $\Pr(x_i=\pm 1) = 1/2$), it cannot hold simultaneously for odd moments:

\begin{lemma}
    There does not exist a random variable $X$ such that 
    $$\E[X] = 0, \E[X^2]=1, \E[X^3] = 1, \E[X^4]=1$$
    hold simultaneously.
\end{lemma}
\begin{proof}
    Applying the Cauchy--Schwarz inequality (Theorem~\ref{thm:cauchy_schwarz})
    to $X$ and $X^2$ gives
    \[
        1
        =
        |\E[X^3]|^2
        \le
        \E[X^2]\E[X^4]
        =
        1.
    \]
    Thus equality holds. Equality in Cauchy--Schwarz implies that
    $X^2=cX$ almost surely for some constant $c$. Taking expectations gives
    \[
        1=\E[X^2]=c\E[X]=0,
    \]
    which is a contradiction.
\end{proof}

The Cauchy-Schwarz inequality is included below for convenience: 
\begin{theorem}[Cauchy-Schwarz inequality]
    \label{thm:cauchy_schwarz}
    For any two random variables $A$, $B$, we have
    $$|\E[AB]|^2 \le \E[A^2]\E[B^2]$$
    where equality holds iff $A = cB$, for some constant $c$. 
\end{theorem}

\section{Dataset Details}
\label{app:dataset_details}

\paragraph{ZINC}
We use the 12,000-molecule subset of ZINC and its 10,000/1,000/1,000 train/validation/test split \cite{irwin2012zinc,dwivedi2023benchmarking}. Molecules are represented by graphs with categorical atom types as node features and categorical bond types as edge features. The task is graph-level regression with the standard normalized penalized-logP as target, evaluated by mean absolute error.

\paragraph{Cycle counting}
We use the synthetic benchmark of \citet{huang2022boosting}, containing 5,000 undirected graphs divided into 1,500/1,000/2,500 training, validation, and test graphs. The node-level targets are the numbers of 3-, 4-, 5-, and 6-cycles containing each node. The graphs have no input node or edge attributes, so every node receives the same constant feature. We report mean absolute error normalized by the target standard deviation.

\paragraph{Directed pattern detection}
We use the synthetic directed-multigraph benchmark of \citet{egressy2024provably}. The dataset has a separate 8,192-node graph for training, validation, and testing. The task is node-level multi-label classification of participation in directed cycles of lengths 2--6, scatter-gather motifs, and bicliques \cite{altman2023realistic}. Node features are constant and there are no edge attributes. Performance is measured by minority-class F1 for each pattern.

\paragraph{DrugOOD}
We use the LBAP-core IC50 task from DrugOOD \cite{ji2023drugood}, under distribution shifts defined by assay, molecular scaffold, and molecule size. Molecules are represented using 39-dimensional atom features, including atom type, hybridization, hydrogen count, and chirality, and 10-dimensional bond features, including bond type, conjugation, ring membership, and stereochemistry. Each setting provides an in-distribution training split and both in- and out-of-distribution validation and test splits. The split sizes depend on the chosen shift.

\paragraph{CIFAR10 superpixels}
We use the CIFAR10 superpixel benchmark of \citet{dwivedi2023benchmarking}, with 45,000/5,000/10,000 train/validation/test graphs. Each image is represented by an 8-nearest-neighbor graph containing 85--150 superpixels. Node features contain continuous RGB values and two-dimensional coordinates, while each edge has a scalar feature derived from spatial distance. The task is graph classification with 10 classes, evaluated by accuracy.

\begin{table}[h]
\centering
\caption{Statistics and input features for the datasets with fixed train/validation/test splits. Average node and edge counts are computed over the full dataset. DrugOOD uses shift-dependent five-way splits and is described in the text.}
\label{tab:dataset-stats}
\resizebox{\linewidth}{!}{%
\begin{tabular}{lcrrll}
\toprule
Dataset & \#Graphs (train/val/test) & Avg. nodes & Avg. edges & Node inputs & Edge inputs \\
\midrule
ZINC & 10{,}000 / 1{,}000 / 1{,}000 & 23.16 & 49.83 & Atom type (28 cat.) & Bond type (3 cat.) \\
Cycle counting & 1{,}500 / 1{,}000 / 2{,}500 & 18.80 & 62.68 & Constant & None \\
Directed pattern detection & 1 / 1 / 1 & 8{,}192.00 & 23{,}630.67 & Constant & None \\
CIFAR10 superpixels & 45{,}000 / 5{,}000 / 10{,}000 & 117.63 & 941.07 & RGB + coordinates (5D) & Spatial distance (1D) \\
\bottomrule
\end{tabular}%
}
\end{table}

\section{Experimental Details}
\label{app:experimental_details}

For ZINC, the structural encoder uses $r=120$ Rademacher probes, a 37-dimensional output encoding, and an eight-layer GIN with hidden dimension 40. The resulting encoding is added to the initial node embeddings and passed to the task GNN: four GINE layers with hidden dimension 128, additive graph pooling, and a three-layer prediction MLP. 

For the cycle counting task, we use the same DIAL structural encoder and no additional base GNN. DIAL uses $r=120$ Rademacher probes, a 37-dimensional output encoding, and a six-layer GIN with hidden dimension 40. This encoding is added to 64-dimensional node embeddings, and node-level predictions are made directly with a two-layer MLP.

For directed pattern detection, both R-PEARL and DIAL are trained as node-level multi-label classifiers using $r=120$ random probes, $K=4$, a 37-dimensional structural encoding, six learned GIN layers with hidden dimension 40, and no additional downstream base GNN. Predictions are made from the resulting node embeddings with a two-layer MLP.

In all experiments, the DIAL structural encoder and the downstream task GNN are trained jointly in an end-to-end fashion. We use the Adam~\citep{kingma2015adam} optimizer with a base learning rate of $0.001$. For directed pattern detection and ZINC, we train for 150 and 2000 epochs, respectively, using a cosine learning-rate schedule with warmup. For the cycle counting task, we train for 2000 epochs and use a ReduceLROnPlateau learning-rate scheduler. Results are averaged over five random seeds for pattern detection and ZINC, and over three random seeds for node-level cycle counting, following the evaluation protocol of the compared methods.

\subsection{Runtime Results}

\begin{table}[t]
\centering
\caption{Per-epoch training and inference time for DIAL on each main benchmark used in the paper. Total training time is a projection, computed as the mean per-epoch training time multiplied by the reported number of training epochs.}
\label{tab:dial-runtime}
\resizebox{1\linewidth}{!}{
\begin{tabular}{lrrrrr}
\toprule
Dataset & Epochs & Train (s/epoch) & Inference (s) & Total train (h) & Peak mem. (GiB) \\
\midrule
ZINC  & 2000 & $15.15 \pm 0.03$ & $0.91 \pm 0.01$ & 8.42 & 9.94 \\
Directed pattern detection & 150 & $0.171 \pm 0.000$ & $0.078 \pm 0.000$ & 0.01 & 10.67 \\
Cycle counting & 2000 & $1.191 \pm 0.009$ & $0.952 \pm 0.011$ & 0.66 & 5.11 \\
\bottomrule
\end{tabular}
}
\end{table}

Table~\ref{tab:dial-runtime} reports the per-epoch training and inference times for DIAL on each benchmark used in the paper. Experiments were run on a single machine with an NVIDIA RTX 4090 (24 GB) GPU and an AMD Ryzen 9 7950X CPU using PyTorch 2.2.2 and CUDA 11.8. Training and inference times are reported as per-epoch means over 22 measured epochs, excluding 3 warmup epochs, and averaged across 3 independent runs.

\section{Additional Experimental Results}
\label{app:additional_experiments}

\subsection{GraphGPS Backbone}

To test whether DIAL is also useful with a hybrid message-passing/Transformer architecture, we combine it with a GraphGPS backbone \citep{rampasek2022recipe}. We use the same dataset, fixed train/validation/test splits, target normalization, evaluation metric, three seeds, and 2,000-epoch training budget as in the original cycle-counting experiment. The GraphGPS and DIAL+GraphGPS variants use identical four-layer GraphGPS backbones with GIN as the local model and four attention heads, and differ only in whether DIAL is included. The original DIAL results are included as a reference.

\begin{center}
\begin{minipage}{\linewidth}
\centering
\captionsetup{hypcap=false}
\captionof{table}{Node-level cycle-counting results with a GraphGPS backbone. We report normalized MAE as mean $\pm$ standard deviation over three seeds; lower is better. Best results are bolded.}
\label{tab:graphgps-cycle-counting}
\setlength{\tabcolsep}{2.5pt}
\scriptsize
\begin{tabular}{lcccc}
\toprule
Method & 3-Cycle & 4-Cycle & 5-Cycle & 6-Cycle \\
\midrule
GraphGPS
& $.2557 \pm .0033$
& $.2027 \pm .0029$
& $.1680 \pm .0031$
& $.1268 \pm .0019$ \\
DIAL (original)
& $\boldsymbol{.0041 \pm .0001}$
& $.0185 \pm .0027$
& $\boldsymbol{.0271 \pm .0005}$
& $.0348 \pm .0030$ \\
DIAL + GraphGPS
& $.0080 \pm .0011$
& $\boldsymbol{.0177 \pm .0006}$
& $.0274 \pm .0012$
& $\boldsymbol{.0292 \pm .0009}$ \\
\bottomrule
\end{tabular}
\end{minipage}
\end{center}

Adding DIAL substantially reduces the error of GraphGPS across all four targets. Relative to the original DIAL architecture, DIAL+GraphGPS has lower mean error on 4- and 6-cycles, similar error on 5-cycles, and higher error on 3-cycles.

\subsection{Ablation of the Post-GIN Correlation}

DIAL correlates the original random probes with their propagated representations at two points in the structural encoder: after linear graph-shift propagation and after the learned GIN layers. The first correlation is the mechanism analyzed theoretically. Since the GIN layers contain nonlinearities, the same analysis does not directly apply to the second correlation. We therefore ablate only the post-GIN correlation, leaving the GIN layers and all other architectural and training choices unchanged.

\begin{center}
\begin{minipage}{\linewidth}
\centering
\captionsetup{hypcap=false}
\captionof{table}{Effect of removing the post-GIN probe correlation on node-level cycle counting. We report normalized MAE as mean $\pm$ standard deviation over three seeds; lower is better. Best results are bolded.}
\label{tab:cycle-correlation-ablation}
\setlength{\tabcolsep}{2.5pt}
\scriptsize
\begin{tabular}{lcccc}
\toprule
Method & 3-Cycle & 4-Cycle & 5-Cycle & 6-Cycle \\
\midrule
DIAL
& $\boldsymbol{.0041 \pm .0001}$
& $\boldsymbol{.0185 \pm .0027}$
& $.0271 \pm .0005$
& $\boldsymbol{.0348 \pm .0030}$ \\
DIAL (no post-GIN corr.)
& $.0046 \pm .0007$
& $.0203 \pm .0020$
& $\boldsymbol{.0260 \pm .0006}$
& $.0371 \pm .0042$ \\
\bottomrule
\end{tabular}
\end{minipage}
\end{center}

\begin{center}
\begin{minipage}{\linewidth}
\centering
\captionsetup{hypcap=false}
\captionof{table}{Effect of removing the post-GIN probe correlation on directed pattern detection. Per-pattern entries are minority-class F1 scores in percent, reported as mean $\pm$ standard deviation over five seeds. Macro F1 is the unweighted mean of the per-pattern means; higher is better. Best results are bolded.}
\label{tab:pattern-correlation-ablation}
\setlength{\tabcolsep}{3.5pt}
\resizebox{\linewidth}{!}{%
\begin{tabular}{lcccccccc}
\toprule
Method & 2-Cycle & 3-Cycle & 4-Cycle & 5-Cycle & 6-Cycle & S-G & B-C & Macro F1 \\
\midrule
DIAL
& $\boldsymbol{91.07 \pm 3.33}$
& $\boldsymbol{82.44 \pm 4.48}$
& $\boldsymbol{79.72 \pm 1.50}$
& $\boldsymbol{72.10 \pm 2.21}$
& $\boldsymbol{74.83 \pm 2.95}$
& $\boldsymbol{78.68 \pm 2.79}$
& $\boldsymbol{66.07 \pm 1.48}$
& $\boldsymbol{77.84}$ \\
DIAL (no post-GIN corr.)
& $70.56 \pm 17.03$
& $71.42 \pm 5.06$
& $75.36 \pm 2.50$
& $69.21 \pm 2.57$
& $73.52 \pm 4.64$
& $75.81 \pm 4.26$
& $65.45 \pm 0.78$
& $71.62$ \\
\bottomrule
\end{tabular}
}
\end{minipage}
\end{center}

Removing the post-GIN correlation has a small and mixed effect on cycle counting: it increases the error for 3-, 4-, and 6-cycles, but slightly reduces it for 5-cycles. On directed pattern detection, the ablation performs worse on every target and reduces macro F1 from $77.84$ to $71.62$, with the largest drops on 2 and 3-cycles. Thus, the post-GIN correlation is empirically useful, particularly for directed pattern detection, even though our estimator analysis applies only to the correlation after linear propagation.

We do not include a variant without either probe correlation because it would closely overlap with the R-PEARL random-feature architecture already evaluated in Tables~\ref{tab:cycle_counting} and~\ref{tab:directed_pattern_detection}.

\subsection{CIFAR10 Superpixels}

We additionally evaluate DIAL on CIFAR10 superpixels, where the input node features contain continuous color and coordinate information. Both methods use 120 probes and concatenate the resulting representation with the task features. All variants use the same four-layer downstream GINE classifier. Table~\ref{tab:cifar10-results} shows that both methods improve over the task-only GINE backbone, while R-PEARL performs better than DIAL on this benchmark.

\begin{center}
\begin{minipage}{\linewidth}
\centering
\captionsetup{hypcap=false}
\captionof{table}{CIFAR10 superpixel graph-classification accuracy (\%) at the checkpoint with the highest validation accuracy. We report mean $\pm$ standard deviation over three seeds; higher is better.}
\label{tab:cifar10-results}
\setlength{\tabcolsep}{8pt}
\begin{tabular}{lc}
\toprule
Method & Test accuracy \\
\midrule
GINE (task only) & $57.55 \pm 0.78$ \\
R-PEARL & $\boldsymbol{61.00 \pm 0.34}$ \\
DIAL & $59.35 \pm 0.53$ \\
\bottomrule
\end{tabular}
\end{minipage}
\end{center}

\subsection{DrugOOD}

Table~\ref{tab:drugood-results} reports in-distribution (ID) and out-of-distribution (OOD) AUROC for DIAL (\(K=4\)) alongside previously reported baselines. DIAL is close to the strongest baseline on the assay split, but performs slightly worse on the scaffold and size splits.

\begin{table}[t]
    \centering
    \caption{
    DrugOOD binding-affinity prediction results (AUROC, \%).
    Baseline results are taken from \citet{kanatsoulis2025learning} and reported as
    mean $\pm$ standard deviation over five seeds.
    Best point estimates for each domain are bolded.
    }
    \label{tab:drugood-results}
    \setlength{\tabcolsep}{3pt}
    \renewcommand{\arraystretch}{0.95}
    \resizebox{\linewidth}{!}{%
    \begin{tabular}{llcccc}
    \toprule
    & & \multicolumn{2}{c}{ID} & \multicolumn{2}{c}{OOD} \\
    \cmidrule(lr){3-4}\cmidrule(lr){5-6}
    Domain & Method & Val & Test & Val & Test \\
    \midrule
    
    \multirow{9}{*}{\textbf{Assay}}
    & No PE
    & $\boldsymbol{92.92 \pm 0.14}$
    & $92.89 \pm 0.14$
    & $71.02 \pm 0.79$
    & $71.68 \pm 1.10$ \\
    & PEG
    & $92.51 \pm 0.17$
    & $92.57 \pm 0.22$
    & $70.86 \pm 0.44$
    & $71.98 \pm 0.65$ \\
    & SignNet
    & $92.26 \pm 0.21$
    & $92.43 \pm 0.27$
    & $70.16 \pm 0.56$
    & $72.27 \pm 0.97$ \\
    & BasisNet
    & $88.96 \pm 1.35$
    & $89.42 \pm 1.18$
    & $71.19 \pm 0.72$
    & $71.66 \pm 0.05$ \\
    & SPE
    & $92.84 \pm 0.20$
    & $\boldsymbol{92.94 \pm 0.15}$
    & $71.26 \pm 0.62$
    & $\boldsymbol{72.53 \pm 0.66}$ \\
    & SPE-8S
    & $92.36 \pm 0.18$
    & $92.62 \pm 0.10$
    & $70.71 \pm 0.47$
    & $71.72 \pm 0.71$ \\
    & R-PEARL
    & $92.71 \pm 0.10$
    & $92.92 \pm 0.12$
    & $70.57 \pm 0.72$
    & $72.24 \pm 0.30$ \\
    & B-PEARL
    & $90.54 \pm 0.89$
    & $90.81 \pm 0.79$
    & $70.53 \pm 0.67$
    & $71.22 \pm 0.42$ \\
    \cmidrule(lr){2-6}
    & DIAL ($K=4$)
    & $92.54\pm 0.09$ & $92.64\pm 0.15$ & $\boldsymbol{71.33\pm 0.64}$ & $72.03\pm 0.49$ \\
    
    \midrule
    
    \multirow{9}{*}{\textbf{Scaffold}}
    & No PE
    & $\boldsymbol{96.56 \pm 0.10}$
    & $87.95 \pm 0.20$
    & $79.07 \pm 0.97$
    & $68.00 \pm 0.60$ \\
    & PEG
    & $95.65 \pm 0.29$
    & $86.20 \pm 0.14$
    & $79.17 \pm 0.97$
    & $69.15 \pm 0.75$ \\
    & SignNet
    & $95.48 \pm 0.34$
    & $86.73 \pm 0.56$
    & $77.81 \pm 0.70$
    & $66.43 \pm 1.06$ \\
    & BasisNet
    & $85.80 \pm 3.75$
    & $78.44 \pm 2.45$
    & $73.36 \pm 1.44$
    & $66.32 \pm 5.68$ \\
    & SPE
    & $96.32 \pm 0.28$
    & $\boldsymbol{88.12 \pm 0.41}$
    & $\boldsymbol{80.03 \pm 0.58}$
    & $\boldsymbol{69.64 \pm 0.49}$ \\
    & SPE-8S
    & $96.44 \pm 0.079$
    & $87.88 \pm 0.45$
    & $79.34 \pm 0.50$
    & $68.72 \pm 0.63$ \\
    & R-PEARL
    & $96.09 \pm 0.32$
    & $88.01 \pm 0.43$
    & $78.72 \pm 0.02$
    & $69.20 \pm 1.00$ \\
    & B-PEARL
    & $96.06 \pm 0.29$
    & $87.56 \pm 0.81$
    & $79.86 \pm 0.58$
    & $69.51 \pm 0.62$ \\
    \cmidrule(lr){2-6}
    & DIAL ($K=4$)
    & $95.84\pm 0.54$ & $86.99 \pm 1.03$ & $78.69 \pm 0.42$ & $67.89\pm 0.79$ \\
    
    \midrule
    
    \multirow{9}{*}{\textbf{Size}}
    & No PE
    & $\boldsymbol{93.78 \pm 0.12}$
    & $93.60 \pm 0.27$
    & $82.76 \pm 0.04$
    & $66.04 \pm 0.70$ \\
    & PEG
    & $92.46 \pm 0.35$
    & $92.67 \pm 0.23$
    & $82.12 \pm 0.49$
    & $66.01 \pm 0.10$ \\
    & SignNet
    & $93.30 \pm 0.43$
    & $93.20 \pm 0.39$
    & $80.67 \pm 0.50$
    & $64.03 \pm 0.70$ \\
    & BasisNet
    & $86.04 \pm 4.01$
    & $85.51 \pm 4.04$
    & $75.97 \pm 1.71$
    & $60.79 \pm 3.19$ \\
    & SPE
    & $92.46 \pm 0.35$
    & $92.67 \pm 0.23$
    & $82.12 \pm 0.49$
    & $66.02 \pm 0.49$ \\
    & SPE-8S
    & $93.68 \pm 0.20$
    & $93.86 \pm 0.12$
    & $83.04 \pm 0.63$
    & $65.74 \pm 2.20$ \\
    & R-PEARL
    & $93.32 \pm 0.34$
    & $\boldsymbol{93.92 \pm 0.20}$
    & $82.09 \pm 0.44$
    & $65.89 \pm 1.30$ \\
    & B-PEARL
    & $93.18 \pm 0.45$
    & $93.29 \pm 0.46$
    & $\boldsymbol{83.14 \pm 0.37}$
    & $\boldsymbol{66.58 \pm 0.67}$ \\
    \cmidrule(lr){2-6}
    & DIAL ($K=4$)
    & $93.46\pm0.17$ & $93.72\pm 0.25$ & $82.07\pm 0.34$ & $65.71\pm 0.23$ \\
    
    \bottomrule
    \end{tabular}%
    }
    \end{table}

\end{document}